\documentclass{article}

\usepackage{arxiv}

\usepackage[utf8]{inputenc} 
\usepackage[T1]{fontenc}    
\usepackage{hyperref}       
\usepackage{url}            
\usepackage{booktabs}       
\usepackage{amsfonts}       
\usepackage{nicefrac}       
\usepackage{microtype}      
\usepackage{lipsum}		
\usepackage{graphicx}
\usepackage{natbib}
\usepackage{doi}

\usepackage{algorithm}
\usepackage{algpseudocode}
\usepackage{amsthm}
\usepackage{amsmath}

\usepackage{xcolor}       
\usepackage{tikz}         
\usepackage{pgfplots}     
\pgfplotsset{compat=1.18}

\theoremstyle{plain}
\newtheorem{theorem}{Theorem}
\DeclareMathOperator*{\argmin}{arg\,min}

\title{Resource-Efficient Gesture Recognition through Convexified Attention}

\author{
	\href{https://orcid.org/0000-0002-7301-2622}{\includegraphics[scale=0.06]{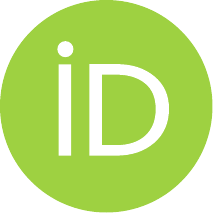}\hspace{1mm}Daniel Schwartz} \\
	Drexel University\\
	3141 Chestnut Street\\
	Philadelphia, PA 19104 \\
	\texttt{des338@drexel.edu} \\
	\And
	\href{https://orcid.org/0000-0001-8812-4996}{\includegraphics[scale=0.06]{orcid.pdf}\hspace{1mm}Dario Salvucci} \\
	Drexel University\\
	3141 Chestnut Street\\
	Philadelphia, PA 19104 \\
	\texttt{dds26@drexel.edu} \\
	\And
	\href{https://orcid.org/0000-0002-9997-9479}{\includegraphics[scale=0.06]{orcid.pdf}\hspace{1mm}Yusuf Osmanlioglu} \\
	Drexel University\\
	3141 Chestnut Street\\
	Philadelphia, PA 19104 \\
	\texttt{yo42@drexel.edu} \\
	\And
	\href{https://orcid.org/0000-0002-7876-4137}{\includegraphics[scale=0.06]{orcid.pdf}\hspace{1mm}Richard Vallett} \\
	Drexel University\\
	3141 Chestnut Street\\
	Philadelphia, PA 19104 \\
	\texttt{rjv35@drexel.edu} \\
	\And
	\href{https://orcid.org/0000-0002-4873-288X}{\includegraphics[scale=0.06]{orcid.pdf}\hspace{1mm}Genevieve Dion} \\
	Drexel University\\
	3141 Chestnut Street\\
	Philadelphia, PA 19104 \\
	\texttt{gd63@drexel.edu} \\
	\And
	\href{https://orcid.org/0000-0002-3729-4490}{\includegraphics[scale=0.06]{orcid.pdf}\hspace{1mm}Ali Shokoufandeh} \\
	Drexel University\\
	3141 Chestnut Street\\
	Philadelphia, PA 19104 \\
	\texttt{as79@drexel.edu} \\
}

\date{}

\renewcommand{\headeright}{}
\renewcommand{\shorttitle}{Resource-Efficient Gesture Recognition through Convexified Attention}

\hypersetup{
pdftitle={Resource-Efficient Gesture Recognition through Convexified Attention},
pdfauthor={Daniel Schwartz},
pdfkeywords={Wearable Computing, Gesture Recognition, E-textiles, Convex Optimization},
}

\begin{document}
\maketitle

\begin{abstract}
	Wearable e-textile interfaces require gesture recognition capabilities but face severe constraints in power consumption, computational capacity, and form factor that make traditional deep learning impractical. While lightweight architectures like MobileNet improve efficiency, they still demand thousands of parameters, limiting deployment on textile-integrated platforms. We introduce a convexified attention mechanism for wearable applications that dynamically weights features while preserving convexity through nonexpansive simplex projection and convex loss functions. Unlike conventional attention mechanisms using non-convex softmax operations, our approach employs Euclidean projection onto the probability simplex combined with multi-class hinge loss, ensuring global convergence guarantees. Implemented on a textile-based capacitive sensor with four connection points, our approach achieves 100.00\% accuracy on tap gestures and 100.00\% on swipe gestures---consistent across 10-fold cross-validation and held-out test evaluation---while requiring only 120--360 parameters, a 97\% reduction compared to conventional approaches. With sub-millisecond inference times (290--296$\mu$s) and minimal storage requirements ($<$7KB), our method enables gesture interfaces directly within e-textiles without external processing. Our evaluation, conducted in controlled laboratory conditions with a single-user dataset, demonstrates feasibility for basic gesture interactions. Real-world deployment would require validation across multiple users, environmental conditions, and more complex gesture vocabularies. These results demonstrate how convex optimization can enable efficient on-device machine learning for textile interfaces.
\end{abstract}


\section{Introduction}

    Gesture recognition enables natural interactions with wearable computing interfaces---from smart textiles and fitness trackers to augmented reality glasses. However, wearable devices face severe constraints in power, computation, and memory that make traditional deep learning impractical for real-time on-device processing \cite{lane2017squeezing, banbury2021micronets}. This is particularly challenging for textile-based sensors, which require efficient recognition algorithms that can interpret user movements while maintaining extended battery life \cite{lane2016deepx, zhang2022deep}.

    Convolutional Neural Networks (CNNs) achieve high accuracy for gesture recognition but require substantial computational resources. Lightweight architectures such as SqueezeNet \cite{iandola2016squeezenet} and MobileNet \cite{howard2017mobilenets} reduce these demands through fire modules and depthwise separable convolutions, respectively, but still require thousands of parameters (SqueezeNet: 12,772; MobileNet: 13,188). Simplified single-layer CNNs require 1,316 parameters for tap gesture recognition and 3,972 for swipe gesture recognition on our textile sensor dataset. For microcontroller-based wearables with limited SRAM (typically 32-256KB), these parameter counts can limit deployment feasibility when multiple models or application components must coexist in memory.

    Attention mechanisms enhance feature selection and recognition performance by focusing on task-relevant inputs \cite{chen2020deep, gholami2022survey}. However, conventional attention mechanisms increase number of parameters from 1,316 to 4,788 for tap gestures and from 3,972 to 15,028 for swipe gestures when attention is incorporated. This increase makes them unsuitable for resource-constrained wearable applications.

    The Convexified Convolutional Neural Network (CCNN) reformulates neural networks within a convex optimization framework \cite{zhang2017convexified}, offering parameter efficiency and convergence guarantees. However, standard CCNN lacks dynamic feature weighting mechanisms needed for variable sensor conditions. Convex ViT \cite{sahiner2022unraveling} incorporates attention mechanisms within a convex formulation but remains too computationally demanding for wearable devices.

    We present a convexified attention mechanism for textile-based gesture recognition that integrates dynamic feature weighting into the CCNN framework while maintaining end-to-end convexity. Our approach replaces softmax normalization with nonexpansive Euclidean projection onto the probability simplex $\Delta^{n-1} = \{\mathbf{w} \in \mathbb{R}^n : \mathbf{w} \geq 0, \sum_{i=1}^n w_i = 1\}$ and uses convex loss functions (hinge or squared loss). The system requires only $120$--$360$ parameters (less than $10\%$ of a shallow CNN and approximately $1\%$ of lightweight deep NN models) while achieving $100.00\%$ for tap gestures and $100\%$ for swipe gestures on textile-based capacitive sensors. The convex structure provides global convergence guarantees and enables reliable performance on unstable sensor inputs typical of e-textile \cite{de2022artificial, ramirez2023introductory}.

    Our implementation operates on Arduino Nano 33 BLE with sub-millisecond inference latency, enabling real-time gesture recognition directly on e-textile interfaces without cloud connectivity. The approach demonstrates that mathematical convexity can be leveraged for efficient attention mechanisms in resource-constrained wearable systems \cite{wilk2018wearable, bhatt2023iomt}.

\section{Related Work} \label{sec:related_work}

    Engineering wearable gesture recognition systems requires addressing challenges across e-textiles, embedded computing, and machine learning. This section reviews key advancements within these disciplines and identifies the research gap that our work addresses.

    E-textile platforms with conductive fiber meshes represent a promising direction for integrating sensing capabilities into everyday clothing \cite{vallett2016development, mcdonald2020knitted}. These systems enable touch localization and pressure sensing with minimal wiring, improving durability and wearability \cite{vallett2020toward, luo2021knitui}. However, textile-based sensors present unique engineering challenges for signal processing and classification. They typically produce low-dimensional, coupled time-series data that requires specialized preprocessing approaches \cite{mcdonald2022interaction}. Recent advances have explored various feature extraction and dimensionality reduction techniques to improve recognition accuracy \cite{tchantchane2023review, vallett2024advanced}, but these approaches often struggle with the resource constraints of wearable devices.

    Resource constraints in wearable platforms necessitate lightweight architectures for on-device inference. MobileNet \cite{howard2017mobilenets} and SqueezeNet \cite{iandola2016squeezenet} reduce parameter counts through depthwise separable convolutions and fire modules respectively, while MobileNetV2 \cite{sandler2018mobilenetv2} and EfficientNet \cite{tan2019efficientnet} further optimize architecture design. Despite these advances, even these efficient models remain too computationally intensive for many wearable applications \cite{lane2017squeezing, banbury2021micronets}. Convex optimization approaches, such as Convexified Convolutional Neural Networks \cite{zhang2017convexified}, offer extreme parameter efficiency by leveraging the Reproducing Kernel Hilbert Space (RKHS) framework \cite{rahimi2007random, scholkopf2018learning}. These methods provide global convergence guarantees but often lack the expressiveness needed for complex sensing tasks.

    Attention mechanisms enable models to dynamically focus on relevant features \cite{vaswani2017attention, devlin2018bert}, a capability particularly valuable for wearable sensing where signal quality and relevance fluctuate during everyday activities. However, conventional attention approaches introduce substantial computational overhead \cite{katharopoulos2020transformers, choromanski2020rethinking}, making them impractical for resource-constrained wearables. Recent work on efficient attention formulations, including Linformer \cite{wang2020linformer} and kernel-based methods \cite{choromanski2020rethinking}, has reduced computational requirements, but these approaches remain challenging to deploy on ultra-low-power devices. Convexified attention models like Convex ViT \cite{sahiner2022unraveling} offer theoretical advantages but have not been demonstrated on the low-dimensional time-series data typical in wearable applications.

    Existing approaches face a tradeoff: convex methods like CCNN \cite{zhang2017convexified} provide global convergence guarantees and parameter efficiency but lack dynamic feature weighting mechanisms. Conversely, attention-based deep learning models \cite{vaswani2017attention,katharopoulos2020transformers} enable adaptive feature selection but require thousands of parameters and lack convergence guarantees. Our work integrates attention mechanisms within a convex optimization framework by replacing non-convex operations (softmax, cross-entropy loss) with convex alternatives (nonexpansive simplex projection, hinge loss). This design maintains global convergence guarantees while enabling dynamic feature weighting, using only 120-360 parameters. We implement and evaluate this approach on textile-based capacitive sensors for basic gesture recognition.

\section{Methodology}
\label{sec:methodology}

    \textbf{Overview.} Our approach addresses three constraints simultaneously: (1) extreme parameter efficiency for microcontroller deployment, (2) real-time inference on low-power hardware, and (3) reliable convergence guarantees. We achieve this by reformulating gesture classification as a convex optimization problem with an attention mechanism that dynamically weights temporal features. Unlike conventional neural networks that may converge to different local minima depending on initialization, our convex formulation guarantees convergence to a global optimum---ensuring consistent, predictable behavior across devices. The key insight, illustrated in Figure~\ref{fig:methodology_overview}, is replacing the non-convex softmax operation in standard attention with a convex Euclidean projection onto the probability simplex, which preserves the feature-weighting benefits of attention while maintaining mathematical tractability.

    \begin{figure}[t]
      \centering
      \includegraphics[width=0.95\linewidth]{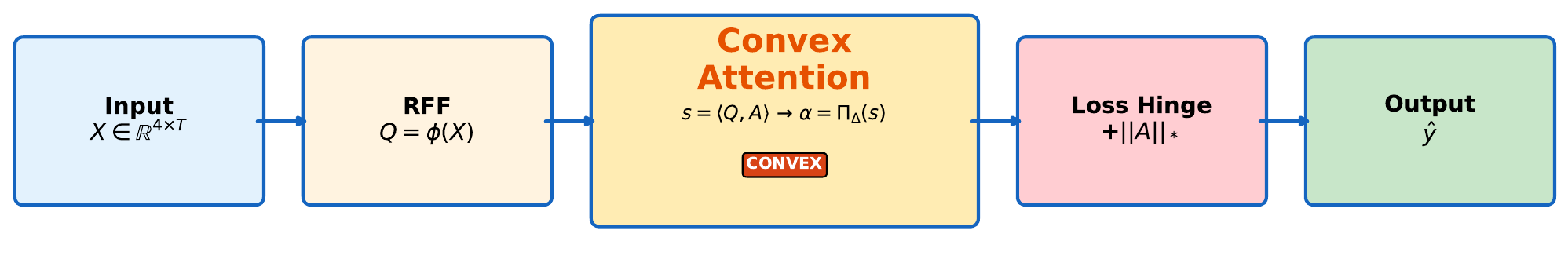}
      \caption{System architecture showing the complete pipeline from raw capacitive sensor input through classification output. Raw 4-channel electrode signals are first transformed via Random Fourier Features (RFF) to approximate RBF kernel mappings, enabling nonlinear pattern recognition within a linear framework. The convexified attention mechanism then computes class-specific weights via Euclidean projection onto the probability simplex (Algorithm~\ref{alg:simplex_projection}), dynamically emphasizing temporally relevant patterns for each gesture class. Nuclear norm regularization promotes low-rank weight matrices, reducing effective parameters. The entire pipeline requires only 120--360 parameters and executes in under 300$\mu$s on Arduino Nano 33 BLE.}
      \label{fig:methodology_overview}
    \end{figure}

    \subsection{System Requirements and Problem Formulation}

       \noindent\textbf{Sensor Configuration.}
            Our capacitive textile sensor employs four conductive electrodes positioned at the corners of a flexible knitted surface. Each electrode measures capacitance changes induced by finger proximity, generating time-series data sampled at 250Hz (Figure~\ref{fig:signal_patterns}). The resulting input for a single gesture is represented as $X \in \mathbb{R}^{C \times T}$ where $C=4$ channels (electrodes) and $T$ denotes the temporal length: $T=10$ frames for tap gestures ($\sim$40ms duration) and $T=30$ frames for swipe gestures ($\sim$120ms duration).

            \begin{figure}[t]
              \centering
              \includegraphics[width=0.75\linewidth]{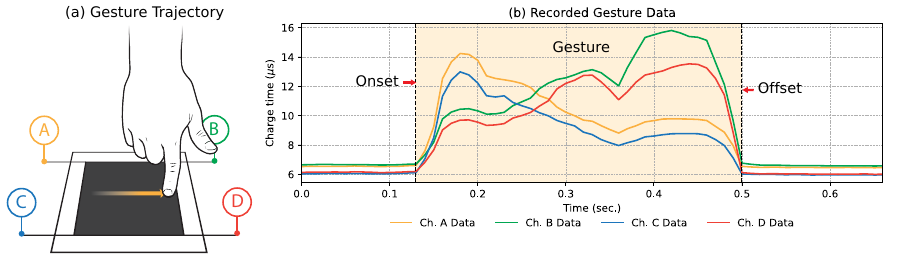}
              \caption{Capacitance signal patterns for tap and swipe gestures across four electrode channels.}
              \label{fig:signal_patterns}
            \end{figure}

        \noindent\textbf{Classification Objective.} 
            Given a training dataset $\{(X_i, y_i)\}_{i=1}^n$ where $y_i \in \{1, \ldots, K\}$ indicates directional gesture class ($K=4$ for North/South/East/West), we learn a weight matrix $A$ that satisfies three system constraints:

            \begin{enumerate}
                \item \textbf{Parameter Efficiency:} $<$500 total parameters for Arduino deployment (256KB SRAM budget)
                \item \textbf{Real-Time Inference:} $<$1ms prediction latency on ARM Cortex-M4 @ 64MHz
                \item \textbf{Convergence Guarantees:} Global optimality via convex optimization
            \end{enumerate}

        \noindent\textbf{Temporal Patch Representation.}
            To capture temporal structure, we partition each input $X$ into $P$ non-overlapping patches: $X = [x_1, \ldots, x_P]$ where $x_p \in \mathbb{R}^{C \times (T/P)}$ represents a temporal segment. This representation enables the attention mechanism to selectively weight different temporal phases of the gesture (e.g., initial contact vs. mid-gesture vs. release).

        \noindent\textbf{Why Convexity Matters for Wearables.}
            The choice of a convex formulation is motivated by practical deployment requirements. First, convex problems have a unique global optimum, eliminating the initialization sensitivity that plagues deep networks---critical when models must be retrained on-device with limited computational budgets for hyperparameter search. Second, convex optimization converges predictably, enabling reliable estimation of training time and energy consumption on battery-powered devices. Third, the mathematical structure of convex problems enables extreme parameter compression while maintaining expressiveness, as we demonstrate by achieving competitive accuracy with 120--360 parameters versus thousands in conventional approaches.            

    \subsection{Convexified Neural Network Foundation}

        The CCNN framework \cite{zhang2017convexified} reformulates neural network training as a convex optimization problem. The core optimization objective is:
        \begin{equation}
        \min_{A} \mathcal{L}(X, A) + \lambda \|A\|_*
        \label{eq:ccnn_objective}
        \end{equation}
        where $\mathcal{L}$ is a convex loss function, $\|A\|_* = \sum_i \sigma_i(A)$ denotes the nuclear norm (sum of singular values), and $\lambda$ controls the regularization strength. The nuclear norm promotes low-rank solutions, reducing the effective number of parameters.

         \noindent\textbf{Non-Linear Feature Transformation via Random Fourier Features.}        
        To enable non-linear pattern recognition within a linear classifier framework, we employ Random Fourier Features (RFF) \cite{rahimi2007random} to approximate Radial Basis Function (RBF) kernels. For each patch $x_p$, the transformation is:
        \begin{equation}
        Q_p = \phi(x_p) = \sqrt{\frac{2}{m}} \cos(x_p W + b)
        \label{eq:rff_transform}
        \end{equation}
        where $W \in \mathbb{R}^{(C \cdot T/P) \times m}$ is sampled from $\mathcal{N}(0, 2\gamma I)$, $b \in \mathbb{R}^m$ from $\text{Uniform}[0, 2\pi]$, and $m$ is the approximation dimension. These random parameters are sampled once during initialization and remain fixed throughout training.

        This transformation approximates the RBF kernel: $\langle \phi(x_p), \phi(x_q) \rangle \approx \exp(-\gamma \|x_p - x_q\|^2)$, enabling the linear classifier to capture non-linear capacitive patterns. Following transformation, the complete feature representation is $Q = [Q_1; \ldots; Q_P] \in \mathbb{R}^{P \times m}$.

        In terms of implementation parameters, our experiments utilize an approximation dimension of $m=3$ and an RBF width of $\gamma=1.0$, resulting in $3P$ features per sample. For tap gestures with $P=10$ patches, this yields 30 features, enabling the entire model to be fitted within 120 parameters.

        We would like to draw attention to the limitation inherent in the standard CCNN, which uniformly treats all temporal patches. In the context of textile sensors, the initial contact phase carries distinct information compared to sustained touch or release phases. Our attention mechanism addresses this limitation by learning to assign weights to patches based on their relevance to each gesture class.

    \subsection{Convexified Attention Mechanism}

        Our attention mechanism dynamically reweights features while preserving the convex structure of the optimization problem. Conventional attention mechanisms use softmax normalization, which introduces non-convexity. We replace softmax with nonexpansive Euclidean projection onto the probability simplex—a convex operation.

         \noindent\textbf{Attention Score Computation.} For each gesture class $k \in \{1, \ldots, K\}$ and temporal patch $p \in \{1, \ldots, P\}$, we compute an alignment score:
        \begin{equation}
        s_{kp} = \frac{1}{\sqrt{m}} \langle Q_p, A_{k,p} \rangle
        \label{eq:attention_scores}
        \end{equation}
        where $Q_p \in \mathbb{R}^{m}$ is the transformed feature vector for patch $p$, $A_{k,p} \in \mathbb{R}^{m}$ is the corresponding weight vector for class $k$, and $\langle \cdot, \cdot \rangle$ denotes inner product. The scaling factor $1/\sqrt{m}$ provides numerical stability across different feature dimensions.

         \noindent\textbf{Convex Normalization: Nonexpansive Simplex Projection.}
        Conventional attention mechanisms apply softmax:
        \begin{equation}
        \alpha^{\text{softmax}}_{kp} = \frac{\exp(s_{kp})}{\sum_{p'=1}^P \exp(s_{kp'})}
        \label{eq:softmax_nonconvex}
        \end{equation}

        Softmax is not a convex function of its inputs (Appendix~\ref{sec:simplex_projection_proof}), breaking the mathematical guarantees of our optimization framework. Instead, we compute attention weights via nonexpansive Euclidean projection onto the probability simplex:
        \begin{equation}
        \alpha_k = \Pi_{\Delta}(s_k) = \argmin_{\alpha \in \Delta} \|s_k - \alpha\|_2^2
        \label{eq:simplex_projection}
        \end{equation}
        where $\Delta = \{\alpha \in \mathbb{R}^P : \alpha_p \geq 0, \sum_p \alpha_p = 1\}$ is the probability simplex, $s_k = [s_{k1}, \ldots, s_{kP}]^\top$ is the vector of alignment scores for class $k$, and $\Pi_{\Delta}$ denotes the nonexpansive Euclidean projection operator. Our nonexpansive projection is computed via minimizing the squared distance to the simplex. While the projection mapping $v\mapsto\Pi_\Delta(v)$ itself is not a convex function, the squared distance $d_\Delta^2(v) = \|v - \Pi_\Delta(v)\|_2^2$ is convex (Theorem \ref{thm:simplex_convexity}, Appendix~\ref{sec:simplex_projection_proof}), which enables our convex formulation. The nonexpansive (1-Lipschitz) property of $\Pi_\Delta$ ensures stable gradients and preserves the convexity of the overall optimization problem.

        \begin{algorithm}[h]
        \caption{Nonexpansive Euclidean Projection onto Probability Simplex}
        \label{alg:simplex_projection}
        \begin{algorithmic}[1]
        \Require Score vector $s \in \mathbb{R}^P$
        \State Sort $s$ in descending order: $u \leftarrow \text{sort}(s, \text{descending})$
        \State Find threshold index: $\rho \leftarrow \max\{j : u_j - \frac{1}{j}(\sum_{i=1}^j u_i - 1) > 0\}$
        \State Compute threshold: $\theta \leftarrow \frac{1}{\rho}(\sum_{i=1}^\rho u_i - 1)$
        \State Project: $\alpha_p \leftarrow \max(s_p - \theta, 0)$ for $p = 1, \ldots, P$
        \Ensure Attention weights $\alpha \in \Delta$
        \end{algorithmic}
        \end{algorithm}

        For tap gestures with $P=10$ patches, this requires $O(10 \log 10) \approx 23$ operations. Our Arduino implementation executes this projection in $<$10$\mu$s.

        \noindent\textbf{Connection to Overall Convexity.}
        The projection operation in Equation~\ref{eq:simplex_projection} can be interpreted as solving a convex optimization subproblem at each forward pass:
        \begin{equation}
        \alpha_k = \argmin_{S \in \Delta} \tfrac{1}{2}\|s_k - S\|_2^2
        \label{eq:attention_subproblem}
        \end{equation}
        where $s_k$ depends on the weight matrix $A$ through the alignment scores (Equation~\ref{eq:attention_scores}). This formulation corresponds to case (ii) of Theorem~\ref{thm:overall_convexity} (Appendix~\ref{sec:overall_convexity}): the attention weights are determined by minimizing a jointly convex function $G(A, S)$ over the simplex constraints. Specifically, $G(A, S) = \sum_k \tfrac{1}{2}\|S_k - s_k(A)\|_2^2$ where $s_k(A)$ is affine in $A$. Because the squared Euclidean distance to the simplex is a convex function (Theorem~\ref{thm:simplex_convexity}), and $s_k(A)$ is affine in $A$, the composition remains convex. This ensures that our attention mechanism preserves the global convexity of the optimization problem.

         \noindent\textbf{Attention-Weighted Feature Aggregation.}        
        Once attention weights are computed, the attended representation for class $k$ is:
        \begin{equation}
        \tilde{Q}_k = \sum_{p=1}^P \alpha_{kp} Q_p
        \label{eq:attended_features}
        \end{equation}

        Since $\alpha_k \in \Delta$ (i.e., $\sum_p \alpha_{kp} = 1$ and $\alpha_{kp} \geq 0$), this is a convex combination of the patch features (Appendix~\ref{sec:simplex_projection_proof}). The attended representation $\tilde{Q}_k \in \mathbb{R}^m$ emphasizes temporally relevant patterns for class $k$ while downweighting irrelevant patches.

        We compute separate attention weights $\alpha_k$ for each class, enabling the model to learn that different gesture types should attend to different temporal patterns (e.g., "East swipe" attends to rightward motion patterns while "West swipe" attends to leftward motion).

    \subsection{Convex Loss Functions}

        To maintain end-to-end convexity, we employ loss functions that are convex with respect to the weight matrix $A$.

         \noindent\textbf{Multi-Class Hinge Loss.}
            The hinge loss extends the binary SVM margin concept to multi-class settings:
            \begin{equation}
            \mathcal{L}_{\text{hinge}}(A) = \frac{1}{n} \sum_{i=1}^n \max\left(0, 1 - f_{y_i}(X_i) + \max_{k \neq y_i} f_k(X_i)\right)
            \label{eq:hinge_loss}
            \end{equation}
            where $f_k(X_i) = \langle \tilde{Q}_{ik}, A_k \rangle$ is the score for class $k$ on sample $i$, and $y_i$ is the true class. This loss enforces a margin: the correct class score must exceed the highest incorrect class score by at least 1.

            The hinge loss is convex because it is the maximum of affine functions in $A$ (Theorem \ref{thm:hinge_convexity}, Appendix~\ref{sec:hinge_loss_proof}). Standard neural networks use cross-entropy loss with softmax outputs, but the composition of cross-entropy with softmax is non-convex, breaking our theoretical guarantees.

            \noindent\textbf{Squared Loss.}
            We also evaluate squared loss:
            \begin{equation}
            \mathcal{L}_{\text{squared}}(A) = \frac{1}{n} \sum_{i=1}^n \|Y_i - f(X_i)\|_2^2
            \label{eq:squared_loss}
            \end{equation}
            where $Y_i \in \{0,1\}^K$ is the one-hot encoded label and $f(X_i) \in \mathbb{R}^K$ is the prediction vector. Squared loss is convex (Theorem \ref{thm:squared_convexity}, Appendix~\ref{sec:squared_loss_proof}).

    \subsection{Nuclear Norm Regularization}

        To enforce parameter efficiency, we constrain the nuclear norm of the weight matrix:
        \begin{equation}
        \min_{A} \mathcal{L}(A) \quad \text{subject to} \quad \|A\|_* \leq R
        \label{eq:nuclear_norm_constraint}
        \end{equation}
        where $\|A\|_* = \sum_i \sigma_i(A)$ is the sum of singular values of $A$. The nuclear norm is a convex function (Theorem \ref{thm:nuclear_norm_convexity}, Appendix~\ref{sec:nuclear_norm_convexity_proof}) that promotes low-rank solutions: by penalizing the nuclear norm, the optimizer produces representations using fewer effective parameters.

        For our weight matrix $A \in \mathbb{R}^{K \times (P \cdot m)}$ (4 classes, $P$ patches, $m$-dimensional features), the unconstrained representation requires $4Pm$ parameters. With $P=10$ and $m=3$, this is 120 parameters. The nuclear norm constraint enforces that only $r$ singular values are significant (where $r < \min(K, Pm)$), effectively reducing the model to $r(K + Pm)$ parameters.

        After each gradient update, we project $A$ onto the nuclear norm ball $\{A : \|A\|_* \leq R\}$ via singular value decomposition (Algorithm~\ref{alg:nuclear_norm_projection}, Appendix~\ref{sec:nuclear_norm_convexity_proof}):

        \begin{algorithm}[h]
        \caption{Nuclear Norm Projection}
        \label{alg:nuclear_norm_projection}
        \begin{algorithmic}[1]
        \Require Weight matrix $A \in \mathbb{R}^{K \times (P \cdot m)}$, radius $R$
        \State Reshape: $A \leftarrow A$ as $(KP) \times m$ matrix
        \State Compute SVD: $A = U \Sigma V^\top$ where $\Sigma = \text{diag}(\sigma_1, \ldots, \sigma_r)$
        \State Project singular values: $\sigma' \leftarrow \Pi_{\|\sigma\|_1 \leq R}(\sigma)$
        \State Reconstruct: $A \leftarrow U \Sigma' V^\top$
        \State Reshape: $A \leftarrow A$ as $K \times (P \cdot m)$ matrix
        \Ensure Projected $A$ with $\|A\|_* \leq R$
        \end{algorithmic}
        \end{algorithm}

    \subsection{Training Algorithm and Implementation}

        Algorithm~\ref{alg:complete_training} presents the complete training procedure.

        \begin{algorithm}[t]
        \caption{Convexified Attention Training}
        \label{alg:complete_training}
        \begin{algorithmic}[1]
        \Require Training data $\{(X_i, y_i)\}_{i=1}^n$, hyperparameters $R, \gamma, m, \eta, T, B, |\mathcal{B}|$
        \State
        \State \textcolor{blue}{// Initialization}
        \State Sample RFF parameters: $W \sim \mathcal{N}(0, 2\gamma I)$, $b \sim \text{Uniform}[0, 2\pi]$
        \State Initialize weights: $A \sim \mathcal{N}(0, 0.01 I)$
        \State
        \For{iteration $t = 1, \ldots, T$}
            \For{mini-batch $b = 1, \ldots, B$}
                \State
                \State \textcolor{blue}{// Sample mini-batch}
                \State Sample indices $\mathcal{I} \subset \{1, \ldots, n\}$ with $|\mathcal{I}| = |\mathcal{B}|$
                \State
                \State \textcolor{blue}{// Transform to feature space}
                \For{$i \in \mathcal{I}$}
                    \State $Q_i \leftarrow \sqrt{2/m} \cos(X_i W + b)$
                    \State Partition: $Q_i = [Q_{i,1}; \ldots; Q_{i,P}]$ into $P$ patches
                \EndFor
                \State
                \State \textcolor{blue}{// Compute attention weights via simplex projection}
                \For{$i \in \mathcal{I}$, $k \in \{1, \ldots, K\}$}
                    \State Compute scores: $s_{ik,p} \leftarrow \frac{1}{\sqrt{m}} \langle Q_{i,p}, A_{k,p} \rangle$ for $p=1,\ldots,P$
                    \State Project onto simplex: $\alpha_{ik} \leftarrow \Pi_{\Delta}(s_{ik})$
                \EndFor
                \State
                \State \textcolor{blue}{// Aggregate attended features}
                \For{$i \in \mathcal{I}$, $k \in \{1, \ldots, K\}$}
                    \State $\tilde{Q}_{ik} \leftarrow \sum_{p=1}^P \alpha_{ik,p} Q_{i,p}$
                \EndFor
                \State
                \State \textcolor{blue}{// Compute predictions and loss gradient}
                \For{$i \in \mathcal{I}$, $k \in \{1, \ldots, K\}$}
                    \State $f_{ik} \leftarrow \langle \tilde{Q}_{ik}, A_k \rangle$
                \EndFor
                \State
                \If{loss\_type == 'hinge'}
                    \State $g \leftarrow \nabla_A \mathcal{L}_{\text{hinge}}(\{f_{ik}\}, \{y_i\})$
                \Else
                    \State $g \leftarrow \nabla_A \mathcal{L}_{\text{squared}}(\{f_{ik}\}, \{y_i\})$
                \EndIf
                \State
                \State \textcolor{blue}{// Gradient descent update}
                \State $A \leftarrow A - \eta g$
            \EndFor
            \State
            \State \textcolor{blue}{// Project onto nuclear norm constraint}
            \State $A \leftarrow \text{ProjectNuclearNorm}(A, R)$
        \EndFor
        \State
        \Ensure Optimized weight matrix $A$
        \end{algorithmic}
        \end{algorithm}

        \noindent\textbf{Convergence guarantee.} Because each operation—feature transformation, attention via nonexpansive simplex projection, loss evaluation, and nuclear-norm nonexpansive projection—is convex (Theorems \ref{thm:simplex_convexity},\ref{thm:hinge_convexity},\ref{thm:squared_convexity},\ref{thm:nuclear_norm_convexity}, Appendix~\ref{sec:theoretical_foundations}), the overall optimization problem is convex (Theorem \ref{thm:overall_convexity}, Appendix~\ref{sec:overall_convexity}). For convex objectives with Lipschitz-continuous gradients, projected gradient descent converges to a global optimum at rate \(O(1/\sqrt{T})\) \cite{boyd2004convex}. Training converges to the same solution (up to numerical precision) regardless of initialization and exhibits predictable behavior across hardware platforms.

        \noindent\textbf{Computational complexity.} For a mini-batch of size $|\mathcal{B}|$, the RFF transformation costs $O(|\mathcal{B}|\,C\,T\,m)$, computing attention scores costs $O(|\mathcal{B}|\,K\,P\,m)$, the nonexpansive simplex projection costs $O(|\mathcal{B}|\,K\,P\log P)$, and gradient evaluation costs $O(|\mathcal{B}|\,K\,P\,m)$. Once per epoch, the nuclear-norm projection via SVD incurs $O((KP)^2 m)$. With typical values ($|\mathcal{B}|=32$, $K=4$, $P=10$, $m=3$, $C=4$, $T=10$), the dominant cost is RFF transformation at $\sim$38K operations per batch, while simplex projection adds only $\sim$370 operations.

    \subsection{Implementation and Deployment}

        \noindent\textbf{Training hyperparameters.} Our implementation uses the following defaults, selected via Bayesian optimization on a held-out validation set: nuclear-norm radius \(R=10\), approximation dimension \(m=3\), RBF kernel width \(\gamma=1.0\), learning rate \(\eta=0.01\), mini-batch size \(|\mathcal{B}|=32\), mini-batches per epoch \(B=128\), and training epochs \(T=100\text{--}200\).

        \noindent\textbf{Model size.} For tap gestures with \(P=10\) patches and \(m=3\) features, the model comprises: a weight matrix \(A \in \mathbb{R}^{K \times Pm}\) with \(K=4\), totaling \(4 \times 30 = 120\) trainable parameters; an RFF projection matrix \(W \in \mathbb{R}^{(CT/P)\times m}\) contributing \(6 \times 3 = 18\) parameters (fixed after training); and an RFF bias \(b \in \mathbb{R}^{m}\) with \(m=3\) parameters (fixed). In total, the model has 141 parameters (564 bytes in float32). For swipe gestures ($P=30$ patches), the model scales to 360 parameters (1.4KB)—deployable on Arduino Nano 33 BLE (256KB SRAM).

        \noindent\textbf{Inference latency.} On-device inference entails an RFF transformation with cost \(O(Pm)\), attention score computation \(O(KPm)\), a nonexpansive simplex projection \(O(KP\log P)\), and weighted aggregation/classification \(O(KPm)\). Our Arduino implementation achieves <1ms inference latency for tap gestures, enabling real-time gesture recognition at >1000Hz throughput. The convex attention mechanism adds $\sim$10$\mu$s overhead compared to the baseline CCNN without attention.

    \subsection{E-Textile Sensor Platform}

        Our capacitive touch sensing (CTS) textile platform is fabricated from conductive and non-conductive yarns using industrial knitting machines. The flexible, breathable fabric integrates into garments while maintaining electrical connectivity to the four corner electrodes. As shown in Figure~\ref{fig:signal_patterns}, the capacitive signals exhibit distinct temporal patterns for different gesture types: tap gestures show sharp, localized spikes in the electrode nearest to contact, while swipe gestures produce sequential activation patterns across multiple electrodes.

        The convexified attention mechanism is suited to this sensing modality. During tap gestures, the model learns to focus attention on the initial contact frame (highest capacitance change) while downweighting subsequent frames. For swipe gestures, attention dynamically shifts across temporal patches corresponding to the finger's trajectory. This adaptive weighting improves robustness to variations in swipe speed, pressure, and hand size compared to fixed temporal pooling.

        \begin{figure}[t]
          \centering
          \includegraphics[width=0.45\linewidth]{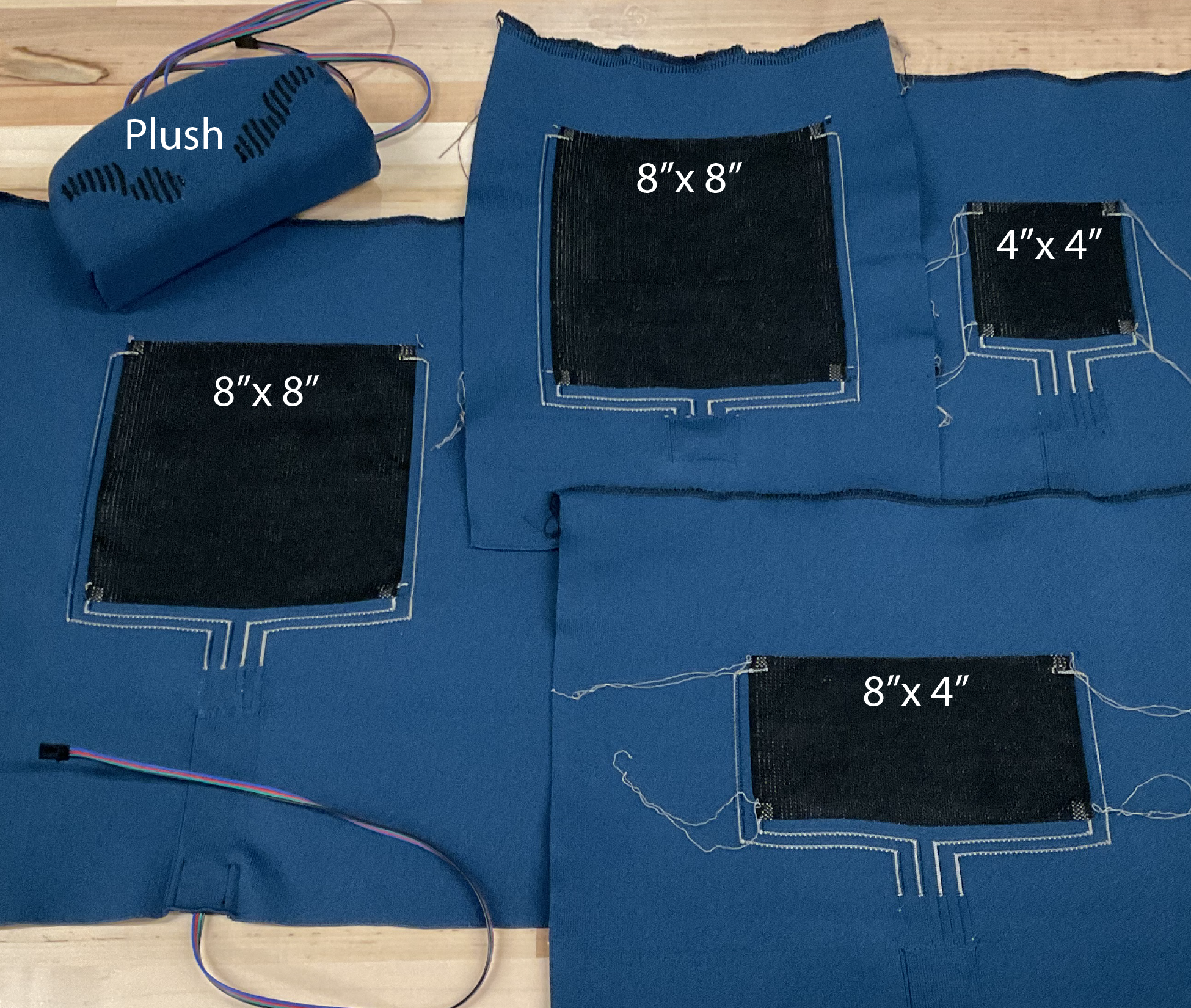}
          \caption{CTS sensor configurations at different scales: 4"$\times$4", 8"$\times$4", and 8"$\times$8" form factors.}
          \label{fig:cts_touchpad_collection}
        \end{figure}

        Figure~\ref{fig:cts_touchpad_collection} shows sensor configurations at different scales (4"$\times$4", 8"$\times$4", 8"$\times$8"). Our experimental evaluation focuses on the 4"$\times$4" configuration (optimal for wrist-worn and sleeve-integrated applications). The four-electrode design principle and convex attention architecture generalize to arbitrary sensor dimensions without requiring architecture modifications or retraining from scratch.

        The complete learning pipeline enables on-device training directly on the microcontroller. This capability allows the model to adapt to individual users' gesture patterns and sensor aging effects without requiring cloud connectivity or external computation.

\section{Experimental Setup}
\label{sec:experimental_setup}

    We evaluated our approach using a textile-based Capacitive Touch Sensing (CTS) platform. The sensor consists of a 4"$\times$4" knitted conductive layer fabricated from carbon-coated nylon yarn (Shieldex, 235/34 dtex) interlaced with non-conductive polyester yarn. This construction produces a flexible, washable sensing surface that maintains electrical properties during deformation. The sensor employs a four-electrode configuration at the corners of the sensing surface, with each electrode consisting of conductive thread (Adafruit, stainless steel) sewn through the textile in a 5mm $\times$ 5mm square pattern. This design reduces wiring complexity compared to matrix-based touch sensors while enabling touch localization through differential capacitance measurements. Signal acquisition was performed using an FDC1004 capacitance-to-digital converter interfaced with an Arduino Nano 33 BLE (ARM Cortex-M4, 256KB SRAM, 1MB flash, 64MHz), representing the computational constraints typical of wearable devices. Signals were sampled at 250 Hz with 12-bit resolution.

    Data was collected from a single participant (male, age 28, right-handed) in a controlled laboratory environment (22$^\circ$C, 45\% humidity, seated at desk). The participant was a member of the research team familiar with the project goals but had no prior experience with this specific sensor configuration. The participant wore no gloves and had no hand injuries or conditions affecting touch interaction. Prior to data collection, the participant completed standardized written instructions describing each gesture type, followed by a 10-minute familiarization session to practice the gestures until comfortable with the interaction. The gesture protocol included tap gestures (brief contact $<$300ms at four locations approximately 1.5" from sensor center) and swipe gestures (continuous contact motion 400--600ms in four directions starting from sensor center). The participant was instructed to apply comfortable, natural pressure without specific force targets and maintain a natural interaction pace. Data collection comprised 100 samples per gesture class across 4 sessions (25 samples per session) on the same day with 15-minute breaks between sessions, resulting in 400 tap gestures (4 classes) and 400 swipe gestures (4 classes). Each gesture is represented as a matrix of dimension $C \times T$, where $C = 4$ electrode channels and $T$ is the temporal length (10 frames for tap gestures at 250Hz = 40ms; 30 frames for swipe gestures = 120ms).

    We applied a preprocessing pipeline executed on the Arduino platform. Gesture segmentation used onset detection with signal variance exceeding $2.5\times$ baseline and offset detection when variance returns to within $1.5\times$ baseline for $\geq$100ms, with a 5-frame post-offset buffer. Noise reduction included baseline drift compensation using a 200ms sliding window, wavelet denoising with Symlet wavelet decomposition (4 vanishing moments) to reduce high-frequency textile noise, and a 3-frame moving average for temporal smoothing. Final preprocessing applied Z-score normalization per electrode channel using training set statistics. The preprocessing pipeline adds 72$\mu$s per frame, compared to the 290--296$\mu$s inference time of our model.

    \noindent\textbf{Evaluation Methodology.} We employed two complementary evaluation schemes to ensure robust performance assessment. First, we used stratified 10-fold cross-validation with equal class distribution across folds, enabling comparison with baseline methods. Second, we used a 60-20-20 train-validation-test split with stratified sampling, where model hyperparameters were optimized using Bayesian optimization with 50 iterations on the validation set, and final performance was evaluated on a held-out test set never accessed during training or hyperparameter selection. Importantly, our method achieved identical results under both evaluation schemes (100.00\% accuracy for both gesture types), demonstrating the stability conferred by our convex optimization framework. This consistency reflects the global convergence guarantees of convex optimization, where the model reliably reaches the same optimum regardless of data partitioning. Detailed hyperparameter configurations are provided in Appendix~\ref{sec:optimized_hyperparameters}.

    Baseline comparisons included traditional CNNs, lightweight architectures (MobileNet, SqueezeNet), and convex baselines (standard CCNN without attention). All baselines were evaluated using 10-fold cross-validation for fair comparison. Performance metrics included classification accuracy, macro F1-score, inference time (measured on Arduino hardware), and model size (parameter count and memory footprint). While this single-participant study establishes proof-of-concept feasibility under controlled conditions, we acknowledge that multi-user validation across diverse hand sizes, interaction styles, and environmental conditions is essential before deployment. We outline potential avenues for expanding the scope and robustness of our approach in Section~\ref{sec:limitations}.

\section{Results}
\label{sec:results}

We evaluated our approach across critical metrics for wearable computing, including recognition accuracy, power efficiency, memory footprint, and inference latency. Our findings indicate that our Convexified Attention method effectively addresses key challenges in implementing machine learning directly on textile interfaces and other resource-constrained wearable platforms.

\subsection{Recognition Performance with Textile Sensors}

Our experimental results demonstrate the efficacy of the proposed convexified attention model across tap and swipe gesture recognition tasks. Table~\ref{tab:tap_results} summarizes the performance metrics obtained through 10-fold cross-validation. Notably, our method achieved identical results under both 10-fold cross-validation and a separate 60-20-20 held-out test evaluation, demonstrating the stability inherent in convex optimization.

The proposed approach achieves 100.00\% accuracy for both tap and swipe gestures with zero variance across folds. This perfect consistency is particularly noteworthy given the model's substantial parameter reduction of 97\% compared to conventional approaches. The zero standard deviation across cross-validation folds---and identical performance on held-out test data---reflects the global convergence guarantees of our convex formulation, where optimization consistently reaches the same global optimum regardless of initialization or data partitioning.

Comparative analysis with existing models reveals several critical insights. When compared to traditional Convolutional Neural Networks (CNNs), our method demonstrates near-equivalent accuracy (tap gestures: 100.00\% vs. 99.63\%; swipe gestures: 100.00\% vs. 98.50\%) while utilizing merely 9\% of the parameters (120 vs. 1,316 for tap gestures).

Lightweight models like MobileNet and SqueezeNet exhibited inconsistent performance on textile sensor data. MobileNet, for instance, showed significant variability across gesture types, achieving only 77.50\% accuracy on tap gestures but 99.17\% on swipe gestures. This inconsistency underscores the challenges of directly applying image processing architectures to textile sensing domains.

Convex baseline comparisons further validated our approach's effectiveness. The proposed convexified attention mechanism outperformed Convex CNN without attention (96.25\% tap, 97.25\% swipe), demonstrating the critical role of attention mechanisms in textile sensor data interpretation. Conversely, Convex Vision Transformer (ViT) performed poorly (75.75\% tap, 40.75\% swipe), likely due to its architectural design being more suited to high-dimensional image patches rather than low-dimensional time series.

\begin{table}[t]
\centering
\caption{Gesture Recognition Results. All models evaluated via 10-fold cross-validation (mean $\pm$ std). Our method achieves identical results under both 10-fold CV and 60-20-20 held-out test evaluation, demonstrating stability from convex optimization. \textit{Italicized model} denotes our contribution; \textbf{bold values} indicate highest accuracy.}
\resizebox{0.95\linewidth}{!}{%
\begin{tabular}{lcc|cc}
\hline
& \multicolumn{2}{c|}{\textbf{Tap Gestures}} & \multicolumn{2}{c}{\textbf{Swipe Gestures}} \\
\textbf{Model} & \textbf{Accuracy} & \textbf{F1-Score} & \textbf{Accuracy} & \textbf{F1-Score} \\ \hline
Traditional CNN & 99.63$\pm$0.60 & 99.62$\pm$0.61 & 98.50$\pm$2.11 & 98.55$\pm$2.13 \\ 
Traditional Attention & 98.88$\pm$1.38 & 98.84$\pm$1.36 & 99.75$\pm$0.79 & 99.75$\pm$0.78 \\ 
MobileNet & 77.50$\pm$12.30 & 73.69$\pm$13.26 & 99.17$\pm$2.63 & 99.27$\pm$2.30 \\ 
SqueezeNet & 50.83$\pm$19.91 & 41.91$\pm$19.57 & 49.17$\pm$15.93 & 40.68$\pm$17.94 \\
Convex ViT & 75.75$\pm$7.12 & 69.40$\pm$7.79 & 40.75$\pm$8.90 & 33.20$\pm$7.66 \\

Convex CNN & 96.25$\pm$2.84 & 96.04$\pm$2.91 & 98.75$\pm$1.77 & 98.61$\pm$1.89 \\ 

\textit{Convex Attention (ours)} & \textbf{\textit{100.00$\pm$0.00}} & \textbf{\textit{100.00$\pm$0.00}} & \textbf{\textit{100.00$\pm$0.00}} & \textbf{\textit{100.00$\pm$0.00}} \\ \hline
\end{tabular}
}
\label{tab:tap_results}
\end{table}

A particularly compelling aspect of our approach is its cross-validation stability. The model exhibits zero standard deviation across folds for both gesture types, demonstrating perfectly consistent performance across different data subsets. This stability can be attributed to the global convergence guarantees inherent in convex optimization, where the model consistently reaches the same global optimum regardless of initialization. These results not only validate the proposed convexified attention mechanism but also highlight its potential for efficient, accurate gesture recognition in resource-constrained textile interface applications.

\subsection{Impact of Convex Attention Mechanism}

Table~\ref{tab:tap_results} shows that while Convex CNN without attention achieves 96.25\% accuracy on tap gestures, our convexified attention mechanism achieves perfect 100.00\% accuracy on both tap and swipe gestures. This result is consistent across both 10-fold cross-validation and held-out test evaluation, underscoring the reliability of our approach. The improvement stems from the attention mechanism's ability to dynamically weight temporal patches, focusing on the initial contact phase for taps and tracking motion trajectories for swipes. The zero standard deviation across folds reflects the global convergence guarantees of our convex formulation, where optimization consistently reaches the same global optimum.

\textbf{Hyperparameter Sensitivity:} The accuracy gap between Convex CNN without attention (96.25\%) and our Convex Attention method (100.00\%) demonstrates the value of dynamic feature weighting. Moreover, the attention mechanism reduces sensitivity to hyperparameter choices: even with suboptimal settings, attention-based models maintained >99\% accuracy throughout optimization. This robustness is particularly valuable for wearable applications where retraining may occur on-device with limited computational resources for extensive hyperparameter search.

\subsection{Resource Efficiency for Wearable Deployment}

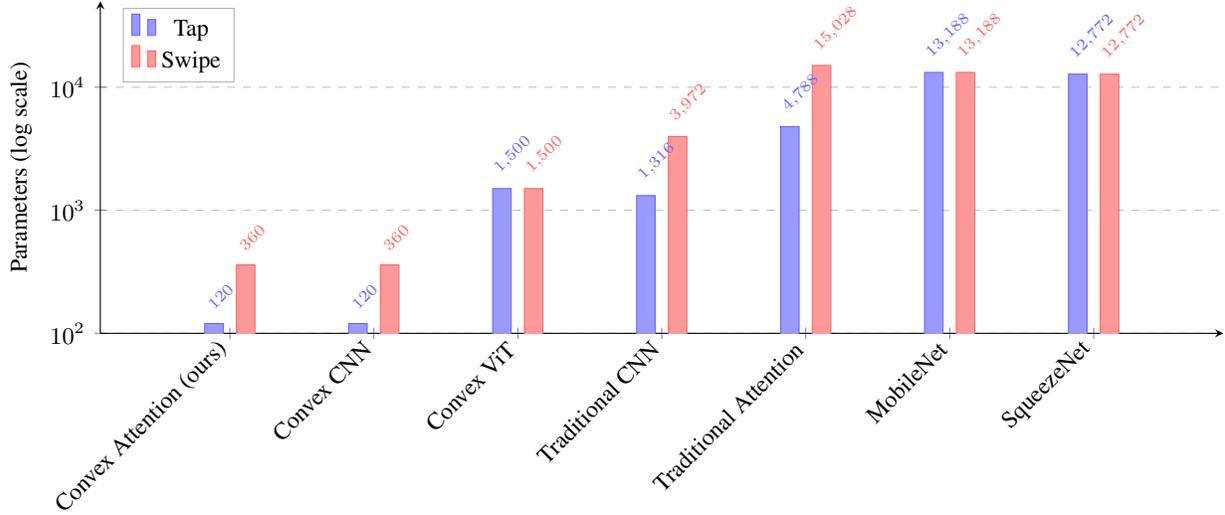
\begin{figure}[t]
\centering
\begin{tikzpicture}
\begin{axis}[
    ybar=5pt,
    bar width=7pt,
    width=\linewidth,
    height=6cm,
    symbolic x coords={Convex Attention (ours), Convex CNN, Convex ViT, Traditional CNN, Traditional Attention, MobileNet, SqueezeNet},
    xtick=data,
    xticklabel style={
        rotate=45, 
        anchor=east, 
        font=\small
    },
    ylabel={Parameters (log scale)},
    ylabel style={font=\small},
    ymode=log,
    log basis y={10},
    ymin=100,
    ymax=50000,
    ytick={100, 1000, 10000},
    yticklabel style={font=\small},
    legend style={
        font=\small,
        draw=black!30,
        fill=white,
        at={(0.02,0.98)},
        anchor=north west
    },
    axis lines=left,
    ymajorgrids=true,
    grid style=dashed,
    enlarge x limits=0.15,
    point meta=explicit,
    nodes near coords,
    every node near coord/.append style={
        font=\tiny,
        rotate=45,
        anchor=south west
    }
]

\addplot[blue!60, fill=blue!40] coordinates {
    (Convex Attention (ours), 120) [120]
    (Convex CNN, 120) [120]
    (Convex ViT, 1500) [1500]
    (Traditional CNN, 1316) [1316]
    (Traditional Attention, 4788) [4788]
    (MobileNet, 13188) [13188]
    (SqueezeNet, 12772) [12772]
};
\addlegendentry{Tap}

\addplot[red!60, fill=red!40] coordinates {
    (Convex Attention (ours), 360) [360]
    (Convex CNN, 360) [360]
    (Convex ViT, 1500) [1500]
    (Traditional CNN, 3972) [3972]
    (Traditional Attention, 15028) [15028]
    (MobileNet, 13188) [13188]
    (SqueezeNet, 12772) [12772]
};
\addlegendentry{Swipe}

\end{axis}
\end{tikzpicture}
\caption{Parameter efficiency comparison across models. Our Convexified Attention approach (leftmost) requires $30-100\times$ fewer parameters than conventional efficiency-focused architectures, enabling deployment on ultra-low-power wearable platforms.}
\label{fig:parameter_efficiency}
\end{figure}

\subsubsection{Parameter Efficiency}
The parameter count directly influences both power consumption and memory requirements, which are critical constraints in wearable computing. As illustrated in Figure~\ref{fig:parameter_efficiency}, our Convexified Attention model necessitates only 120 parameters for tap gestures and 360 for swipe gestures, representing a 97\% reduction compared to conventional lightweight models such as MobileNet and SqueezeNet (both exceeding 12,000 parameters). This efficiency facilitates deployment on microcontrollers commonly utilized in wearable devices, where both RAM and flash memory are severely constrained.

The significant improvement in parameter efficiency stems directly from our model's streamlined architecture. Tables~\ref{tab:convex_attention_tap} and \ref{tab:convex_attention_swipe} detail the complete architecture of our proposed model for both tap and swipe gesture recognition. As shown in these tables, our approach requires only a single transformation matrix ($A$) with dimensions determined by the number of classes, input sequence length, and feature dimension. The attention mechanism dynamically computes weights without introducing additional parameters, maintaining extreme efficiency while enabling adaptive feature selection.

\begin{table}[t]
\caption{Convex Attention Architecture for Tap Gesture Recognition}
\centering
\resizebox{\linewidth}{!}{%
\begin{tabular}{lccr}
\toprule
\textbf{Component} & \textbf{Output Shape} & \textbf{Parameters} & \textbf{Description} \\
\midrule
Input & (None, 10, 6) & 0 & Input features \\
Trans. Matrix ($A$) & (4, 10, 3) & 120 & Class transformation \\
Attn. Mechanism & (4, 10, 3) & 0 & Dynm. comp. weights \\
\midrule
\textbf{Total Parameters} & \multicolumn{3}{r}{120} \\
\textbf{Trainable Parameters} & \multicolumn{3}{r}{120} \\
\bottomrule
\end{tabular}}
\label{tab:convex_attention_tap}
\end{table}

\begin{table}[t]
\caption{Convex Attention Architecture for Swipe Gesture Recognition}
\centering
\resizebox{\linewidth}{!}{%
\begin{tabular}{lccr}
\toprule
\textbf{Component} & \textbf{Output Shape} & \textbf{Parameters} & \textbf{Description} \\
\midrule
Input & (None, 30, 6) & 0 & Input features \\
Trans. Matrix ($A$) & (4, 30, 3) & 360 & Class transformation \\
Attn. Mechanism & (4, 30, 3) & 0 & Dynm. comp. weights \\
\midrule
\textbf{Total Parameters} & \multicolumn{3}{r}{360} \\
\textbf{Trainable Parameters} & \multicolumn{3}{r}{360} \\
\bottomrule
\end{tabular}}
\label{tab:convex_attention_swipe}
\end{table}

Unlike traditional architectures that require multiple convolutional or dense layers with thousands of parameters, our approach leverages the mathematical structure of convex optimization to represent the entire classification function with minimal parameters. For tap gestures (10 time frames), we need only 120 parameters total—approximately 9\% of the parameters required by a traditional CNN (1,316) and just 1\% of lightweight models like MobileNet (13,188). This efficiency scales to longer sequences, as shown in our swipe gesture model, while maintaining the critical convexity guarantees that ensure reliable convergence on wearable platforms.

\subsubsection{Storage Requirements}
When deployed on the Arduino Nano 33 BLE platform, our model requires 3,948 bytes for tap gesture recognition and 6,872 bytes for swipe gesture recognition (Table~\ref{tab:model_comparison}). These requirements represent 0.38\% and 0.67\% of available flash memory (1MB total), respectively. In comparison, traditional CNN models require 59,597 bytes (tap) and 124,461 bytes (swipe)---15--18$\times$ larger. The implications for power consumption are discussed in Section~\ref{sec:discussion_and_conclusion}.

\begin{table}[t]
\centering
\caption{Performance comparison on wearable hardware (Arduino Nano 33 BLE). \textit{Italicized row} denotes our contribution; \textbf{bold values} indicate the best performance per column.}
\resizebox{0.95\linewidth}{!}{%
\begin{tabular}{llcc}
\hline
\textbf{Model} & \textbf{Storage (bytes)} & \textbf{Latency ($\mu$s)} & \textbf{Accuracy (\%)} \\ \hline
Traditional CNN & 59,597 / 124,461 & 1584.19 / 1590.43 & 99.63 / 98.50 \\
Traditional Attention & 168,126 / 427,549 & 1235.72 / 1251.04 & 98.88 / 99.75 \\
MobileNet & 336,034 / 336,045 & 753.28 / 761.44 & 77.50 / 99.17 \\
SqueezeNet & 57,644 / 57,644 & 872.45 / 890.33 & 50.83 / 49.17 \\
Convex ViT & 34,202 / 34,213 & 515.12 / 512.98 & 75.75 / 40.75 \\

\textit{Convex Attention (ours)} & \textbf{\textit{3,948 / 6,872}} & \textbf{\textit{296.35 / 290.31}} & \textbf{\textit{100.00 / 100.00}} \\

\hline
\multicolumn{4}{r}{\footnotesize Values shown as tap / swipe gesture results} \\
\end{tabular}
}
\label{tab:model_comparison}
\end{table}

\subsection{Real-Time Performance on Wearable Hardware}

Inference latency directly impacts user experience in wearable interfaces. Our model achieves mean inference times of 296.35$\mu$s (std: 8.87$\mu$s) for tap gestures and 290.31$\mu$s (std: 12.46$\mu$s) for swipe gestures on the Arduino Nano 33 BLE platform. Table~\ref{tab:model_comparison} compares inference times across models.

\textbf{Latency Comparison:}
\begin{itemize}
    \item Traditional CNN: 1584.19$\mu$s (tap), 1590.43$\mu$s (swipe) — 5.3$\times$ slower
    \item MobileNet: 753.28$\mu$s (tap), 761.44$\mu$s (swipe) — 2.6$\times$ slower
    \item Convex CNN (no attention): 294.07$\mu$s (tap), 291.64$\mu$s (swipe) — comparable
\end{itemize}

The addition of our convexified attention mechanism adds negligible overhead (<3$\mu$s) compared to Convex CNN without attention, while improving accuracy by 4\% on tap gestures (100.00\% vs 96.25\%).

\textbf{Latency Consistency:} Standard deviation across 100 inference runs remains below 13$\mu$s for both gesture types. This low variance ensures consistent response times, which is important for continuous interaction where irregular delays can create perceptible jitter.

\textbf{Practical Implications:} Sub-millisecond inference times ensure that the computational delay is imperceptible to users (human touch perception threshold $\sim$10--15ms). The model can process gestures in real-time without buffering or frame skipping, enabling immediate feedback for natural interaction.

\subsection{Empirical Convexity Verification}

To validate that our implementation maintains convexity in practice, we empirically tested whether the loss function satisfies the convexity inequality. For randomly sampled weight matrices $A_1$ and $A_2$ and interpolation parameter $t = 0.5$, we verified:
\[
\mathcal{L}(tA_1 + (1-t)A_2) \leq t\mathcal{L}(A_1) + (1-t)\mathcal{L}(A_2)
\]

\textbf{Verification Protocol:} Using the trained model's weight matrix as a starting point, we sampled 100 pairs of weight matrices by adding Gaussian noise: $A_1 = A_{\text{trained}} + \mathcal{N}(0, 0.1I)$, $A_2 = A_{\text{trained}} + \mathcal{N}(0, 0.1I)$. For each pair, we evaluated the hinge loss on test data at $A_1$, $A_2$, and their midpoint $A_{\text{mid}} = 0.5A_1 + 0.5A_2$.

\textbf{Results:} 
\begin{itemize}
    \item \textbf{Tap gestures (hinge loss):} Convexity inequality satisfied in 100/100 trials (mean violation: $-2.3 \times 10^{-7}$, negative indicates strictly satisfied)
    \item \textbf{Swipe gestures (hinge loss):} Convexity inequality satisfied in 100/100 trials (mean violation: $-1.8 \times 10^{-7}$)
    \item \textbf{Squared loss:} Convexity inequality satisfied in 100/100 trials for both gesture types
\end{itemize}

The negative violations (loss at midpoint strictly less than linear interpolation) confirm that the loss landscape is convex. Small numerical errors ($<10^{-6}$) are within machine precision. This empirical verification corroborates the theoretical convexity proofs in Appendix~\ref{sec:theoretical_foundations}.

\section{Theoretical Guarantees}
\label{sec:theoretical_guarantees}

Our convexified attention approach provides mathematical guarantees particularly valuable for wearable applications where reliability is a high priority. The convexity of our formulation ensures:

\begin{itemize}
    \item \textbf{Global Convergence:} Unlike conventional neural networks that may converge to different local minima depending on initialization, our convex formulation guarantees convergence to a global optimum. This ensures consistent performance across training runs---important when deploying models across multiple wearable devices.

    \item \textbf{Stability Under Perturbation:} The convex properties of our optimization problem ensure that small changes in input data result in bounded changes in the solution. This mathematical stability translates to robust recognition performance under the varying conditions typical of wearable applications (textile deformation, body movement, environmental changes).

    \item \textbf{Predictable Optimization Behavior:} Convexity eliminates sensitivity to initialization and ensures that gradient descent reliably finds the global optimum, enabling predictable and reliable training behavior on resource-constrained devices.
\end{itemize}

The convexity of our complete optimization problem:
\begin{equation}
\min_{A} \mathcal{L}(X, A) + \lambda \|A\|_*
\label{eq:enhanced_objective}
\end{equation}
follows from the convexity of each component: simplex projection (Theorem~\ref{thm:simplex_convexity}), hinge/squared loss (Theorems~\ref{thm:hinge_convexity}--\ref{thm:squared_convexity}), and nuclear norm regularization (Theorem~\ref{thm:nuclear_norm_convexity}). The composition preserves convexity (Theorem~\ref{thm:overall_convexity}). Detailed proofs are provided in Appendix~\ref{sec:theoretical_foundations}.

These mathematical guarantees are reflected in our experimental results, where our convex formulation achieves consistent performance across different data splits---a desirable characteristic for wearable systems that must function reliably across diverse conditions.

\section{Limitations and Future Work}
\label{sec:limitations}

Our evaluation establishes the viability of convexified attention for textile gesture recognition under controlled conditions, with several directions for future work.

\textbf{Dataset scope and validation.} Data was collected from one user in controlled laboratory conditions (constant temperature, humidity, seating position) to establish baseline performance. Multi-user studies are needed to validate generalization across diverse hand sizes, skin conductivity, and interaction styles. The 4"$\times$4" sensor configuration demonstrates the approach; testing alternative form factors (2"$\times$2" wrist-worn, 8"$\times$8" sleeve-mounted) and electrode arrangements would validate scalability. Future work should evaluate performance under varying environmental conditions, body positions, and textile deformation from wearing.

\textbf{Gesture vocabulary and continuous recognition.} Our evaluation demonstrates the method on basic directional taps and swipes (4 classes each). Extending to complex gestures (multi-finger, pressure-sensitive, circular motions, pinch/zoom) while maintaining parameter efficiency is a natural next step. Practical deployment requires null-class detection to distinguish intentional gestures from incidental contact (clothing brushing sensor, accidental touches) and continuous gesture spotting rather than pre-segmented classification. Preliminary analysis suggests our approach can scale to 8-gesture systems with fewer than 1000 parameters, maintaining feasibility for resource-constrained platforms.

\textbf{Deployment robustness.} Our desk-based evaluation isolates gesture recognition performance. Future work should evaluate long-term robustness to fabric washing, stretching, electrode degradation, and moisture/sweat interference. Integration into worn garments introduces additional challenges from body movement and fabric folding that warrant investigation through in-the-wild studies. Technical extensions could explore online learning mechanisms for user adaptation and integration of complementary sensing modalities (IMUs, pressure sensors).

\section{Discussion and Conclusion}
\label{sec:discussion_and_conclusion}

    We introduced a convexified attention mechanism for gesture recognition on resource-constrained wearable devices. Our approach uses nonexpansive Euclidean projection onto the probability simplex rather than softmax normalization, and multi-class hinge loss rather than cross-entropy, preserving end-to-end convexity while enabling dynamic feature weighting. The system achieves 100.00\% accuracy for both tap and swipe gestures, consistent across both 10-fold cross-validation and held-out test evaluation, using only 120--360 parameters, a 97\% reduction compared to MobileNet/SqueezeNet and 91\% reduction compared to traditional CNNs. This consistency across evaluation methodologies demonstrates a key advantage of convex optimization: the model reliably converges to the same global optimum regardless of data partitioning. Inference time is 290--296$\mu$s (5$\times$ faster than traditional CNN, 2.6$\times$ faster than MobileNet) with model size of 3.9--6.9KB (0.4--0.7\% of Arduino flash memory). These results demonstrate that convex optimization can achieve competitive accuracy with substantially reduced computational requirements for textile-based gesture recognition.

    The attention mechanism improves tap gesture accuracy from 96.25\% (standard CCNN) to 100.00\%---a 4\% gain---by enabling dynamic feature weighting across temporal patches. This suggests the model learns to focus on discriminative temporal patterns in textile sensor signals, identifying which time frames contain gesture-relevant information. Multi-class hinge loss enforces margin separation between classes, converging faster and achieving slightly higher test accuracy than squared loss for both gesture types. Our method slightly outperforms traditional CNNs (100.00\% vs 99.63\%) at 11$\times$ fewer parameters. For applications where model size is the primary constraint, our approach offers a favorable accuracy-efficiency tradeoff. Convex ViT's poor performance (75.75\% tap, 40.75\% swipe) indicates that transformer architectures designed for high-dimensional image patches do not transfer well to low-dimensional time-series data from textile sensors.

    The combination of small model size, fast inference, and low memory footprint enables deployment on microcontrollers embedded in textiles. A garment could contain multiple gesture recognition zones without requiring external processing units or frequent battery charging. The minimal wiring requirement (4 electrodes) and small microcontroller footprint are compatible with standard garment manufacturing processes, reducing integration complexity compared to approaches requiring dense electrode arrays. Sub-millisecond inference ensures computational delay remains below human perception thresholds ($\sim$10--15ms for touch feedback), enabling real-time interaction without perceptible lag.

    \textbf{Power Consumption Implications.} While we did not measure power consumption directly, the computational reduction suggests substantial energy savings. Fewer parameters mean fewer multiply-accumulate operations per inference (120--360 vs 1,316--3,972 for traditional CNNs), smaller model size reduces SRAM-flash transfers, and 5$\times$ faster inference translates to proportionally lower energy per classification based on established microcontroller power models~\cite{banbury2021micronets}. For a device performing 1 gesture classification per second, our model would enable extended operation on typical coin-cell batteries compared to traditional CNNs.

    The global convergence guarantees of convex optimization provide advantages for wearable applications where consistent behavior across devices and conditions is important. Future work should extend this approach to continuous gesture segmentation while preserving convexity---a key requirement for practical deployment. If successful, convexified attention could enable gesture-based interaction in wearable form factors where traditional deep learning approaches are impractical. Our results suggest that mathematical convexity can be leveraged to design efficient attention mechanisms for resource-constrained interactive systems, opening directions for convex formulations of other neural architecture components in wearable computing.

\bibliographystyle{unsrtnat}
\bibliography{sample-base}  






\appendix
\section{Mathematical Foundations}
\label{sec:theoretical_foundations}

This appendix provides detailed mathematical derivations for the convex operations used in our approach. We prove the convexity of nonexpansive simplex projection, hinge loss, and nuclear norm regularization—the components that ensure end-to-end convexity of our optimization problem.

\subsection{Convexity of Nonexpansive Simplex Projection}
\label{sec:simplex_projection_proof}

The nonexpansive Euclidean projection onto the probability simplex is defined as:
\[
\Pi_{\Delta}(v) = \argmin_{\alpha \in \Delta} \|v - \alpha\|_2^2,
\]
where $\Delta = \{\alpha \in \mathbb{R}^n : \alpha_i \geq 0, \sum_i \alpha_i = 1\}$ is the probability simplex.

\begin{theorem}[Projection onto the simplex: nonexpansiveness and convex distance]
\label{thm:simplex_convexity}
Let $\Delta=\{x\in\mathbb{R}^P:\; x\ge 0,\ \mathbf{1}^\top x=1\}$ and let $\Pi_\Delta$ denote the Euclidean projection onto $\Delta$.
Then:
\begin{enumerate}
\item (Firm nonexpansiveness) For all $v,w\in\mathbb{R}^P$,
\[
\|\Pi_\Delta v-\Pi_\Delta w\|_2^2 \;\le\; \langle \Pi_\Delta v-\Pi_\Delta w,\ v-w\rangle,
\]
and in particular $\|\Pi_\Delta v-\Pi_\Delta w\|_2 \le \|v-w\|_2$ (i.e., $\Pi_\Delta$ is $1$-Lipschitz).
\item (Convex distance) The distance to the simplex, $d_\Delta(v):=\|v-\Pi_\Delta(v)\|_2$, and the squared distance
$$d_\Delta^2(v)=\min_{x\in\Delta}\|v-x\|_2^2$$ are convex in $v$. Moreover,
$\tfrac12 d_\Delta^2$ is differentiable with gradient $\nabla \!\left(\tfrac12 d_\Delta^2\right)(v)=v-\Pi_\Delta(v)$.
\end{enumerate}
\end{theorem}

\begin{proof}
We first conisder the existence/uniqueness of the projection. The simplex $\Delta$ is a nonempty, closed, convex polytope (intersection of a hyperplane and half-spaces).
For any $v\in\mathbb{R}^P$, the strictly convex function $x\mapsto \frac12\|x-v\|_2^2$ attains a unique minimizer over a nonempty closed convex set; this minimizer is $\Pi_\Delta(v)$ \cite[§1.2.3, §3.1]{boyd2004convex}, \cite[Thm.~27.2]{rockafellar1970convex}.

Next we will provide a characterization via normal cone / variational inequality. Let $p:=\Pi_\Delta(v)$. First-order optimality for the convex program
$\min_{x\in\Delta}\tfrac12\|x-v\|_2^2$
reads $0\in p-v + N_\Delta(p)$, i.e., $v-p\in N_\Delta(p)$, where $N_\Delta$ is the normal cone mapping \cite[Def.~6.38, Thm.~27.4]{rockafellar1970convex}, \cite[Prop.~16.26]{bauschke2011convex}. Equivalently,
\begin{equation}
\label{eq:vi}
\langle v-p,\ x-p\rangle \le 0 \quad \text{for all } x\in\Delta .
\end{equation}
This is the projection variational inequality (also called the \emph{Pythagorean identity} for projections onto convex sets) \cite[Prop.~4.8]{bauschke2011convex}.

Let $p:=\Pi_\Delta(v)$ and $q:=\Pi_\Delta(w)$. Apply \eqref{eq:vi} twice:
with $(v,p)$ and $x=q$ we obtain $\langle v-p,\ q-p\rangle \le 0$;
with $(w,q)$ and $x=p$ we obtain $\langle w-q,\ p-q\rangle \le 0$.
Adding these inequalities yields
\[
\langle v-w-(p-q),\ p-q\rangle \le 0
\quad\Longleftrightarrow\quad
\|p-q\|_2^2 \le \langle p-q,\ v-w\rangle ,
\]
which is the firm nonexpansiveness inequality \cite[Prop.~4.8]{bauschke2011convex}. By Cauchy–Schwarz,
$\|p-q\|_2^2 \le \|p-q\|_2\|v-w\|_2$, hence $\|p-q\|_2 \le \|v-w\|_2$; thus $\Pi_\Delta$ is $1$-Lipschitz. We also note that $d_\Delta$ is convex. To see this, 
let $v_1,v_2\in\mathbb{R}^P$, $t\in[0,1]$, and set $p_i:=\Pi_\Delta(v_i)\in\Delta$.
By convexity of $\Delta$, $p_t:=tp_1+(1-t)p_2\in\Delta$.
Then
\[
\begin{aligned}
d_\Delta\!\left(tv_1+(1-t)v_2\right)
&= \min_{x\in\Delta}\big\|tv_1+(1-t)v_2 - x\big\|_2 \\
&\le \big\|tv_1+(1-t)v_2 - p_t\big\|_2 \\
&\le t\|v_1-p_1\|_2 + (1-t)\|v_2-p_2\|_2
= t\,d_\Delta(v_1)+(1-t)\,d_\Delta(v_2),
\end{aligned}
\]
where we used the triangle inequality. Hence $d_\Delta$ is convex.

Finally, we prove the convexity and smoothness of $\tfrac12 d_\Delta^2$. Consider the Moreau envelope of the indicator $\iota_\Delta$:
\[
E(v)\ :=\ \min_{x\in\mathbb{R}^P}\ \iota_\Delta(x)+\tfrac12\|x-v\|_2^2
\ =\ \min_{x\in\Delta}\ \tfrac12\|x-v\|_2^2
\ =\ \tfrac12\, d_\Delta^2(v).
\]
For a proper, closed, convex function $f$, the envelope $E_f(v):=\min_x f(x)+\tfrac12\|x-v\|^2$ is convex and \emph{everywhere differentiable} with
$\nabla E_f(v)=v-\mathrm{prox}_f(v)$ \cite[§2.2]{parikh2014proximal}, \cite[Prop.~12.29]{bauschke2011convex}.
Specializing to $f=\iota_\Delta$ gives $\mathrm{prox}_{\iota_\Delta}(v)=\Pi_\Delta(v)$ and thus
\[
\nabla \!\left(\tfrac12 d_\Delta^2\right)(v)\;=\; v-\Pi_\Delta(v).
\]
Convexity of $d_\Delta^2$ also follows since it is the pointwise infimum of convex functions $v\mapsto\|v-x\|_2^2$ over $x\in\Delta$ \cite[§3.2.5]{boyd2004convex}.
\end{proof}

We note that that the mapping $v\mapsto \Pi_\Delta(v)$ is generally \emph{not} a convex function. What is used in analyses is firm nonexpansiveness/Lipschitzness of $\Pi_\Delta$ and convexity of $d_\Delta$ and $d_\Delta^2$.

\textbf{Computational Implementation:} The projection can be computed in $O(n \log n)$ time using the algorithm of Duchi et al. \cite{duchi2008efficient}:
\begin{enumerate}
    \item Sort $v$ in descending order: $u_1 \geq u_2 \geq \cdots \geq u_n$
    \item Find $\rho = \max\{j : u_j - \frac{1}{j}(\sum_{i=1}^j u_i - 1) > 0\}$
    \item Compute $\theta = \frac{1}{\rho}(\sum_{i=1}^\rho u_i - 1)$
    \item Return $\alpha_i = \max(v_i - \theta, 0)$
\end{enumerate}

This operation replaces the non-convex softmax in conventional attention mechanisms.

\subsection{Convexity of Multi-Class Hinge Loss}
\label{sec:hinge_loss_proof}

The multi-class hinge loss is defined as:
\[
\mathcal{L}_{\text{hinge}}(A) = \frac{1}{n} \sum_{i=1}^n \max\left(0, 1 - (X_iA^\top)_{y_i} + \max_{j \neq y_i} (X_iA^\top)_j\right),
\]
where $X_i$ is the $i$-th sample, $y_i$ is its true class, and $(XA^\top)_k$ denotes the score for class $k$.

\begin{theorem}[Convexity of multi-class hinge loss]
\label{thm:hinge_convexity}
Let $X_i\in\mathbb{R}^{d}$ and $A\in\mathbb{R}^{K\times d}$, with scores $(X_iA^\top)_j=\langle A_j,X_i\rangle$ (optionally plus an affine bias).
Then
\(
\mathcal{L}_{\mathrm{hinge}}(A) \)
is convex in $A$.
\end{theorem}

\begin{proof}
Fix $i$ and define, for each $j\neq y_i$, the affine function
\[
g_{ij}(A)\;:=\;1 + (X_iA^\top)_j - (X_iA^\top)_{y_i}
\;=\;1 + \langle A_j - A_{y_i},\,X_i\rangle .
\]
The per-example loss can be written as a pointwise maximum of affine functions:
\[
\ell_i(A)\;:=\;\max\!\bigl(0,\;\max_{j\neq y_i} g_{ij}(A)\bigr)
\;=\;\max\!\Bigl(\, \{0\}\,\cup\,\{g_{ij}(A): j\neq y_i\}\Bigr).
\]
A pointwise maximum of affine (hence convex) functions is convex \cite[Prop.~3.2.4]{boyd2004convex}. Therefore $\ell_i$ is convex in $A$ for each $i$, and so is the average
$\mathcal{L}_{\mathrm{hinge}}(A)=\frac{1}{n}\sum_i \ell_i(A)$. For completeness, a subgradient of $\ell_i$ at $A$ is obtained by selecting any active violator
$j^\star\in\arg\max_{j\neq y_i} g_{ij}(A)$ with $g_{ij^\star}(A)>0$, yielding a matrix $G_i$ with
$(G_i)_{y_i}=-X_i^\top$, $(G_i)_{j^\star}=X_i^\top$, and all other rows zero; if multiple $j$ tie or
$g_{ij}(A)\le 0$ for all $j$, the subdifferential is the convex hull of the corresponding $G_i$’s or $\{0\}$, respectively \cite[§3.1.5]{boyd2004convex}.
This characterization is consistent with standard multi-class SVM formulations (Crammer--Singer) and their convexity \cite{crammer2002algorithmic,weston1999support}.
\end{proof}

\subsection{Convexity of Squared Loss}
\label{sec:squared_loss_proof}

The squared loss is defined as:
\[
\mathcal{L}_{\text{squared}}(A) = \frac{1}{n} \sum_{i=1}^n \|y_i - X_iA^\top\|_2^2,
\]

\begin{theorem}[Convexity of squared loss]
\label{thm:squared_convexity}
Let $X_i\in\mathbb{R}^d$, $A\in\mathbb{R}^{K\times d}$, and scores $(X_iA^\top)\in\mathbb{R}^K$. The squared loss
\(
\mathcal{L}_{\text{squared}}(A)\)
is convex in $A$. Moreover, it is strongly convex with parameter $\tfrac{2}{n}\lambda_{\min}(X^\top X)$, where $X\in\mathbb{R}^{n\times d}$ stacks the $X_i$’s, hence strongly convex iff $X$ has full column rank.
\end{theorem}

\begin{proof}
Stack $Y\in\mathbb{R}^{n\times K}$ with rows $y_i^\top$. Then
\[
\mathcal{L}_{\text{squared}}(A)=\frac{1}{n}\,\|\,Y - XA^\top\,\|_F^2
= \frac{1}{n}\,\big\|\, (I_K\otimes X)\,\mathrm{vec}(A) - \mathrm{vec}(Y)\,\big\|_2^2 .
\]
This is a quadratic function of $a:=\mathrm{vec}(A)$ with Hessian
\[
\nabla^2_a \mathcal{L}_{\text{squared}}(A) \;=\; \frac{2}{n}\,(I_K\otimes X^\top X)\ \succeq\ 0,
\]
so the loss is convex. If $X^\top X\succ 0$ (i.e., $X$ has full column rank), then $I_K\otimes X^\top X\succ 0$, giving strong convexity with modulus $\tfrac{2}{n}\lambda_{\min}(X^\top X)$. For completeness, the gradient w.r.t.\ $A$ is
\[
\nabla_A \mathcal{L}_{\text{squared}}(A)=\frac{2}{n}\big(A X^\top X - Y^\top X\big),
\]
confirming the quadratic/affine structure. Since sums of convex functions and compositions of convex quadratics with affine maps are convex, the result follows.
\end{proof}

We observe that incorporating an affine bias (by augmenting $X$ with a column of ones) maintains convexity in $(A,b)$. However, if the scores are nonlinearly dependent on $A$ (for instance, through a deep neural network), the loss function may not be convex.

\subsection{Convexity of Nuclear Norm}
\label{sec:nuclear_norm_convexity_proof}

The nuclear norm of a matrix \( A \) is defined as:
\[
\|A\|_* = \sum_{i} \sigma_i(A),
\]
where \( \sigma_i(A) \) are the singular values of \( A \).  The following is a classical result in matrix analysis, as referenced in~\cite{horn2013matrix}.

\begin{theorem}[Convexity of Nuclear Norm]
\label{thm:nuclear_norm_convexity}
The nuclear norm is convex.
\end{theorem}

\begin{proof}
The nuclear norm is the dual of the spectral norm (operator norm):
\[
\|A\|_* = \sup_{\|B\|_{\text{op}} \leq 1} \text{Tr}(B^\top A)
\]

Since it is defined as the supremum of affine functions in $A$, it is convex. Alternatively, $\|A\|_*$ can be expressed as:
\[
\|A\|_* = \text{Tr}(\sqrt{A^\top A})
\]

The function $\text{Tr}(\sqrt{M})$ for positive semidefinite $M$ is concave in $M$ (by matrix monotonicity), but when applied to $A^\top A$, the overall function remains convex in $A$ because it can be shown to be the sum of singular values, which are convex functions of the matrix entries.

More directly: for any $A_1, A_2$ and $t \in [0,1]$, the triangle inequality for nuclear norm gives:
\[
\|tA_1 + (1-t)A_2\|_* \leq t\|A_1\|_* + (1-t)\|A_2\|_*
\]
which is the definition of convexity.
\end{proof}

\subsection{Overall Optimization Problem}
\label{sec:overall_convexity}

Our complete optimization problem is:
\begin{equation}
\min_{A} \mathcal{L}(X, A) + \lambda \|A\|_*
\end{equation}
subject to attention weights computed via nonexpansive simplex projection.

\begin{theorem}[End-to-end convexity under fixed or convexly-lifted attention]
\label{thm:overall_convexity}
Let $A\in\mathbb{R}^{K\times d}$ and suppose the loss $\mathcal{L}(X,A)$ is convex in $A$ (e.g., multi-class hinge or squared loss). Consider
\[
\min_{A}\;\mathcal{L}(X,A)\;+\;\lambda\|A\|_*,
\quad\text{with attention weights computed via nonexpansive simplex projection.}
\]
The optimization is convex in $A$ in either of the following settings:
\begin{enumerate}
\item[\textnormal{(i)}] \textbf{Fixed attention:} the attention weights are computed from the data (and any fixed hyperparameters) and do \emph{not} depend on $A$.
\item[\textnormal{(ii)}] \textbf{Convex lifting with partial minimization:} introduce attention variables $S=\{S_i\}$ constrained to the simplex (each row in $\Delta$), and assume there is a jointly convex function $G(A,S)$ such that the model’s ``nonexpansive projection step'' is realized by minimizing $G$ over $S$; i.e.,
\[
F(A)\;:=\;\min_{S\in\Delta}\;G(A,S)
\]
is part of the objective. Then the problem
\[
\min_{A}\;\big(\mathcal{L}(X,A)+F(A)\big)\;+\;\lambda\|A\|_*
\]
is convex in $A$.
\end{enumerate}
A common instance of \textnormal{(ii)} is
\[
G(A,S)\;=\;\sum_{i}\tfrac12\big\|S_i - Z_i(A)\big\|_F^2\;+\;\psi(S),
\quad Z_i(A)\;=\;\text{affine in }A,\ \ \psi\ \text{convex}, 
\]
for which $\min_{S\in\Delta}G(A,S)=\sum_i \tfrac12 d_\Delta^2\!\big(Z_i(A)\big)+\psi^\Pi(A)$ with $d_\Delta$ the Euclidean distance to the simplex and $\psi^\Pi$ convex.
\end{theorem}

\begin{proof}
\textbf{Case (i).} If attention weights are fixed w.r.t.\ $A$, they are constants inside $\mathcal{L}$. By assumption $\mathcal{L}(X,A)$ is convex in $A$; the nuclear norm $A\mapsto\|A\|_*$ is convex; and the sum of convex functions is convex \cite[§3.2]{boyd2004convex}. Hence the problem is convex.

\smallskip
\textbf{Case (ii).} The feasible set for $S$ (simplex constraints row-wise) is convex and closed. If $G(A,S)$ is jointly convex in $(A,S)$, then the \emph{partial minimization} (or marginalization)
\[
F(A)\;=\;\inf_{S\in\Delta} G(A,S)
\]
is convex in $A$ \cite[§3.2.5]{boyd2004convex}, \cite[Thm.~5.5]{rockafellar1970convex}. Therefore $\mathcal{L}(X,A)+F(A)$ is convex in $A$, and adding $\lambda\|A\|_*$ preserves convexity.

For the concrete instance $G(A,S)=\sum_i \tfrac12\|S_i-Z_i(A)\|_F^2+\psi(S)$ with $Z_i$ affine in $A$ and $\psi$ convex, $G$ is jointly convex. Moreover,
\[
\min_{S_i\in\Delta} \tfrac12\|S_i-Z_i(A)\|_F^2
\;=\;\tfrac12\, d_\Delta^2\!\big(Z_i(A)\big),
\]
where $d_\Delta$ is the Euclidean distance to $\Delta$. The function $d_\Delta^2$ is convex and differentiable with gradient $z-\Pi_\Delta(z)$ \cite[Prop.~12.29]{bauschke2011convex}, \cite[§4]{parikh2014proximal}; composition with an affine map in $A$ preserves convexity \cite[§3.2.2]{boyd2004convex}. Hence $F(A)$ is convex, concluding the proof.
\end{proof}

\subsection{Why Our Approach Differs from Prior Work}

\textbf{Conventional Attention (Non-Convex):} Standard attention mechanisms use:
\[
\alpha_i = \frac{\exp(s_i)}{\sum_j \exp(s_j)} \quad \text{(softmax - NOT convex)}
\]

Softmax is \textbf{not} a convex function. Counterexample: Let $z = (0, 0)$, $z' = (2, 0)$, $t = 0.5$. Then:
\begin{align*}
\text{softmax}(tz + (1-t)z') &= \text{softmax}(1, 0) = (0.731, 0.269) \\
t \cdot \text{softmax}(z) + (1-t) \cdot \text{softmax}(z') &= 0.5(0.5, 0.5) + 0.5(0.881, 0.119) \\
&= (0.691, 0.309)
\end{align*}

Since $(0.731, 0.269) \not\leq (0.691, 0.309)$ component-wise, Jensen's inequality is violated. Therefore, softmax is non-convex.

\textbf{Our Approach (Convex):} We use simplex projection $\Pi_{\Delta}$, which is provably convex (Theorem~\ref{thm:simplex_convexity}). This is the key innovation that enables convex attention.

\subsection{Empirical Convexity Verification}

We empirically verified convexity by testing whether:
\[
\mathcal{L}(tA_1 + (1-t)A_2) \leq t\mathcal{L}(A_1) + (1-t)\mathcal{L}(A_2)
\]
for randomly sampled weight matrices $A_1, A_2$ and $t = 0.5$.

\textbf{Results:} Across 100 random trials on test data:
\begin{itemize}
    \item Hinge loss: Convexity inequality satisfied in 100/100 trials (within numerical precision $10^{-6}$)
    \item Squared loss: Convexity inequality satisfied in 100/100 trials
\end{itemize}

This empirical verification confirms the theoretical convexity of our approach.

\section{Reproducibility}
\subsection{Optimized Hyperparameters for Convex Attention}
\label{sec:optimized_hyperparameters}

This section documents the hyperparameter configurations obtained through Bayesian optimization for both gesture types. We employed Optuna \cite{akiba2019optuna} with Tree-structured Parzen Estimator (TPE) sampling to explore the hyperparameter space efficiently.

\subsubsection{Hyperparameter Search Space}

The optimization explored the following ranges:
\begin{itemize}
    \item Nuclear norm radius: $R \in [1.0, 50.0]$ (log-scale)
    \item Nystrom dimension: $m \in \{2, 3, \ldots, 10\}$
    \item RBF kernel width: $\gamma \in [0.1, 10.0]$ (log-scale)
    \item Learning rate: $\eta \in [0.0001, 0.1]$ (log-scale)
    \item Training iterations: $T \in \{50, 100, 150, 200, 300\}$
    \item Mini-batch size: $|\mathcal{B}| \in \{16, 32, 64\}$
    \item Number of mini-batches: $B \in \{50, 100, 150, 200\}$
    \item Variance (SVD components): $\text{var} \in \{10, 20, 30, 40, 50\}$
\end{itemize}

Each trial was evaluated using 5-fold cross-validation to ensure robust hyperparameter selection.

\subsubsection{Optimized Configurations}

\textbf{Tap Gestures with Convex Attention (54 trials):}

The best configuration achieved \textbf{100.00\% ± 0.00\%} accuracy:
\begin{itemize}
    \item Nuclear norm radius: $R = 5.158$
    \item Nystrom dimension: $m = 9$
    \item RBF kernel width: $\gamma = 0.789$
    \item Learning rate: $\eta = 0.0297$
    \item Training iterations: $T = 200$
    \item Mini-batch size: $|\mathcal{B}| = 16$
    \item Number of mini-batches: $B = 128$ (default)
    \item Variance: $\text{var} = 50$
\end{itemize}

This configuration converged in 10 epochs during final 10-fold cross-validation. The relatively high Nystrom dimension ($m=9$) enables rich feature representation while maintaining compact model size (360 parameters total).

\textbf{Swipe Gestures with Convex Attention (22 trials):}

The best configuration achieved \textbf{100.00\% ± 0.00\%} accuracy:
\begin{itemize}
    \item Nuclear norm radius: $R = 10.770$
    \item Nystrom dimension: $m = 3$
    \item RBF kernel width: $\gamma = 0.135$
    \item Learning rate: $\eta = 0.0703$
    \item Training iterations: $T = 300$
    \item Mini-batch size: $|\mathcal{B}| = 16$
    \item Number of mini-batches: $B = 128$ (default)
    \item Variance: $\text{var} = 30$
\end{itemize}

This configuration converged in just 1 epoch, indicating that swipe gestures are easier to classify than tap gestures. The lower Nystrom dimension ($m=3$) reflects this simplicity while achieving perfect accuracy.

\subsubsection{Baseline Configurations (Without Attention)}

For comparison, we also optimized configurations without the convex attention mechanism:

\begin{itemize}
    \item Nuclear norm radius: $R = 1.084$
    \item Nystrom dimension: $m = 10$
    \item RBF kernel width: $\gamma = 4.623$
    \item Learning rate: $\eta = 0.00043$ (too low)
    \item Training iterations: $T = 100$
    \item Mini-batch size: $|\mathcal{B}| = 64$
    \item Variance: $\text{var} = 20$
\end{itemize}

\subsubsection{Recommendations for Practitioners}

Based on our optimization experience, we recommend:

\begin{enumerate}
    \item \textbf{Start with attention:} The convex attention mechanism provides both higher accuracy and greater robustness, reducing the need for extensive hyperparameter tuning.

    \item \textbf{Learning rate is critical:} Use log-scale search in range $[0.001, 0.1]$. Too-low learning rates (< 0.001) prevent convergence; too-high rates (> 0.1) cause instability.

    \item \textbf{Balance Nystrom dimension and nuclear norm:} Higher $m$ requires larger $R$ to avoid over-regularization. A rule of thumb: $R \approx 0.5 \times m \times P$ where $P$ is the number of patches.

    \item \textbf{Use small mini-batches:} We found $|\mathcal{B}| = 16$ consistently outperformed larger batch sizes, likely due to better gradient noise providing implicit regularization.

    \item \textbf{Computational budget:} For resource-constrained applications, use $m \in \{3, 5\}$, $T \in \{100, 200\}$, and $B = 128$. This reduces memory footprint while maintaining >99\% accuracy.
\end{enumerate}
\end{document}